\documentclass{article}
\usepackage{ijcai22}

\usepackage{times}
\usepackage{soul}
\usepackage{url}
\usepackage[hidelinks]{hyperref}
\usepackage[utf8]{inputenc}
\usepackage[small]{caption}
\usepackage{graphicx}
\usepackage{amsmath}
\usepackage{amsmath}
\usepackage{amstext}
\usepackage{amssymb}
\usepackage{amsfonts}
\usepackage{amsthm}
\usepackage{booktabs}
\usepackage{algorithm}
\usepackage{algorithmic}
\usepackage{float}

\newcommand{\mycite}[1]{\citeauthor{#1}~[\citeyear{#1}]}
\newcommand{\code}[1]{\texttt{#1}}

\newcommand{\abs}[1]{\left|{#1}\right|}

\newcommand{\x}{\cdot}

\newcommand{\argmin}[1]{\underset{#1}{\mathrm{argmin}}}

\newcommand{\set}[1]{\left\{{#1}\right\}}
\newcommand{\tuple}[1]{\left({#1}\right)}

\newcommand{\floor}[1]{\left\lfloor{#1}\right\rfloor}

\newcommand{\InSet}{\textsc{InSet}}
\newcommand{\MaxSizeSplit}{\textsc{MaxSizeSplit}}
\newcommand{\Information}{\textsc{Information}}
\newcommand{\ExpSizeSplit}{\textsc{ExpSize\-Split}}

\newcommand{\MostParts}{\textsc{MostParts}}

\newcommand{\Total}{\textsc{Total}}

\newcommand{\Expected}{\textsc{Expected}}

\newcommand{\TurnsNeeded}{\textsc{TurnsNeeded}}

\newcommand{\UG}[1]{\textsc{UG}\left(#1\right)}

\newcommand{\Bound}[1]{\textsc{Bound}\left(#1\right)}
\newcommand{\MaxSplits}[1]{\textsc{MaxSplits}\left(#1\right)}
\newcommand{\MinTotal}[1]{\textsc{MinTotal}\left(#1\right)}
\newcommand{\nSplits}[1]{\textsc{nSplits}\left(#1\right)}

\newtheorem{theorem}{Theorem}
\newtheorem{definition}{Definition}
\newtheorem{corollary}{Corollary}
\newtheorem{lemma}{Lemma}
\newtheorem{proposition}{Proposition}
\newtheorem{property}{Property}

\newenvironment{hproof}{%
  \proof}{\endproof}

\title{On Optimal Strategies for Wordle and General Guessing Games}

\author{
Michael Cunanan
\And
Michael Thielscher \\
\affiliations
School of Computer Science and Engineering, University of New South Wales\\
\emails
cunananm2000@gmail.com, mit@unsw.edu.au
}

\begin{document}

\maketitle

\begin{abstract}
    The recent popularity of Wordle has revived interest in guessing games. We develop a general method for finding optimal strategies for guessing games while avoiding an exhaustive search. Our main contributions are several theorems that build towards a general theory to prove the optimality of a strategy for a guessing game. This work is developed to apply to any guessing game, but we use Wordle as an example to present concrete results.
\end{abstract}

\section{Introduction}\label{ch:intro}

Mastermind is a guessing game that has been studied extensively in the past
\cite{knuth,stuckman2005mastermind,doerr2016playing,glazik2021bounds}.
Such work has not seemed to be carried over to other guessing games, however. Our vision is to have AI agents learn how to approach any game of this kind, similar to a general-game-playing setting \cite{genese:intern,C:5}. To do this we supply human intelligence to guide this area of research; this paper aims to do just that for general guessing games. We also aim to add mathematical rigour to the study of meta-reasoning in guessing games, such as in \cite{filman1983metalanguage}, or to aid in developing predicates for grounded languages such as in \cite{thomason2016learning}.

The timing of this publication coincides with the recent popularity of the online game Wordle \cite{wordle}, which we will use for our example guessing game of choice.
Wordle is a word game that was published in October 2021. Since then, it has gained significant popularity, with over 300,000 daily users in January 2022 \cite{wordleusers}. There has been widespread interest in the general community for an optimal approach to the game, with several websites making unsupported claims to have determined the best strategy.

The game itself is a guessing game in which players must deduce a hidden word using clues that the game gives in response to the player's guesses, with a fixed limit of 6 guesses allowed. The exact details of these clues and the structure of the game will be explored in further detail in the next section.

The popularity of Wordle has also caused several variants to appear, including with
\begin{itemize}
    \item Different word sets (e.g. Bardle \cite{bardle}, FFXIV\-rdle \cite{ffxivrdle})
    \item Multiple games at the same time (e.g. Dordle \cite{dordle}, Tridle \cite{tridle}, Sexaginta-quattuordle  \cite{64ordle})
    \item Completely different forms of input (e.g. Heardle \cite{heardle}, Chessle \cite{chessle}).
\end{itemize}
As such, the focus of this paper lies in guessing games in general, but we will use Wordle as the main example throughout.


Our main contribution is a series of theorems that build towards a general method to determine if a strategy is optimal or not, without the need for an exhaustive search. These formal results can also be used to find an optimal strategy. The theorems we present are generalized to work for any guessing games to automatically find strategies and prove their optimality. We specifically demonstrate using these theorems to show the Wordle strategy found by our framework is optimal. We also present a method of determining the next optimal guess, which to our knowledge, is a novel approach.

The remainder of the paper is organized as follows. In the next section, we recapitulate the basic components of guessing games in general, including Wordle, and we recapitulate known heuristics from the literature on Mastermind. In Section~\ref{ch:finding_good_strats}, we show how to combine heuristics to search for good strategies. In Section~\ref{sec:proving}, we present novel and general theorems by which a strategy can be proved optimal without an exhaustive search. In the section that follows, we demonstrate using the general method and theorems on Wordle and some of its variants. We conclude in Section~\ref{sec:conclusion}.%
\footnote{This is an extended version, with full proofs and additional examples in the appendix, of a paper accepted at IJCAI 2023.}
\section{Background}\label{ch:background}
\subsection{Guessing Games}
In this section, we define exactly what we consider to be general guessing games, following similar definitions by \mycite{minguesses} and \mycite{focardi2012guessing}. \citeauthor{minguesses} refer to a guessing game as an `interactive knowledge transfer model', but for the sake of readability we will use the term `guessing game'.

In a \textbf{guessing game}, we have two parties: a \textbf{learner} and a \textbf{teacher}. The teacher's goal is to communicate some \textbf{secret} $s$ that is initially hidden from the learner. The learner submits a \textbf{guess} $g$ to the teacher, to which the teacher responds with some \textbf{response} $r$. The teacher's responses are to be used as clues by the learner to deduce what $s$ is. The teacher computes responses using an \textbf{answering function} $a$; this function is known to both parties. These communications continue until the teacher responds with the \textbf{affirmative response} $r^*$, at which point we say the learner has learnt the secret and has won the game. The following definition summarizes the components of a general guessing game.
\begin{definition}[Guessing game]\label{def:guessing_game}
A guessing game can be uniquely represented as a tuple $(G, S, R, r^*, a)$, where
\begin{align*}
    G &= \text{Set of allowable guesses}\\
    S &= \text{Set of allowable secrets, with } S \subseteq G\\
    R &= \text{Set of possible responses, with } \abs{R} > 1\\
    r^* &= \text{Affirmative response, with } r^* \in R\\
    a &= \text{Answering function of type } G \times S \to R,\\
    &\quad \quad \text{with } \forall g \in G \ \forall s  \in S: \ a(g,s) = r^* \iff g = s
\end{align*}
and $G$, $S$ and $R$ are finite sets.
\end{definition}

We assume that $S$ is known to the learner, though exactly which element of $S$ is the secret is not known. This may not be true in practice for human players, but we do this because any guessing game should have a well-defined domain of secrets that an AI certainly could use.

\subsection{Wordle}
In Wordle\footnote{Wordle has several variants. In this paper we consider Wordle in its default mode.}, the player (learner) must deduce a common 5-letter English word chosen by the computer (teacher). Similar to the well-known guessing game Mastermind, the player's guesses are met with colour-coded responses to guide them toward the answer. The secret in Wordle changes daily.

We provide an example play of Wordle in Figure~\ref{fig:wordle} using the game from March 22 2022. The player's first guess was \code{TARES}. The computer assigned a grey colour to \code{T}, \code{A}, \code{R}, \code{E}, and so none of those letters appear in the secret word. The letter \code{S} however was assigned a yellow colour, which indicated that \code{S} does appear in the secret, but not in its current position (i.e.\ the $5^{th}$ position). The player's next guess was \code{SPOIL}. Now \code{S} and \code{O} are assigned green, which indicates \code{S} and \code{O} appear in those positions in the secret word. The player uses these colour encodings as responses from the computer to determine what to guess next throughout the game. Wordle ends when the player receives an all-green encoding, as seen in the final row.

To represent colours we use $\code{0} = \text{Grey}$, $\code{1} = \text{Yellow}$ and $\code{2} = \text{Green}$ in the following definition.

\begin{figure}
    \centering
    \includegraphics[width=0.8\linewidth]{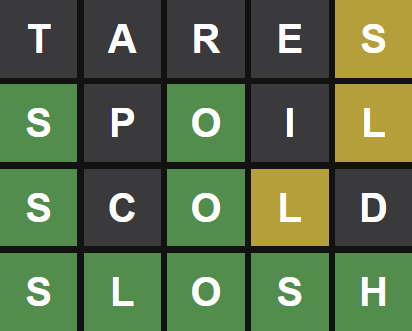}
    \caption{March 22 2022 Wordle puzzle, completed in 4 turns.}
    \label{fig:wordle}
\end{figure}

\begin{definition}[Wordle]\label{def:wordle}
Wordle is a guessing game $(G_W, S_W, R_W, r^*_W, a_W)$ according to Definition~\ref{def:guessing_game}, where
\begin{align*}
    G_W &= \text{All 5 letter words, curated by the developer.}\\
    S_W &= \text{Common 5 letter words decided by the developer.}\\
    R_W &= \set{\code{00000}, \code{00001},\dots,\code{22220},\code{22222}}\\
    r^*_W &= \code{22222}\\
    a_W &= \text{See the above paragraph and Figure~\ref{fig:wordle} for example.}
\end{align*}

\end{definition}



$G_W$ and $S_W$ are publicly known sets and can be found in Wordle's source code. It is worth noting however that this set has been altered a few times since the game's creation. Our research into Wordle strategies was initially conducted on Wordle's original sets of guesses and secrets (before 15th February 2022), and so we describe our work using these sets, with ${\abs{G_W} = 12972}$ and ${\abs{S_W} = 2315}$. Importantly however, the results presented in this paper could easily be replicated for the updated Wordle sets, or any other word set in general, as we will see.

\subsection{Strategies}\label{subsec:strats}
In this section, we formalize how we intend agents to play guessing games by defining strategies; as well as how we intend to compare the performance of strategies.

As a learner plays a guessing game, they should be using the previously submitted guesses and the corresponding received responses to make informed decisions about what guess to submit next. We may capture this learned information using \textbf{candidates}.

\begin{definition}[Candidates]\label{def:candidates}
Suppose in a guessing game $(G,S,R,r^*,a)$ the guesses and responses so far are $\tuple{(g_1,r_1),\dots,(g_n,r_n)}$, i.e.\ guess $g_i$ was met with response $r_i$. The \textbf{candidate set} $C$ is defined as
\begin{align*}
    C = \bigcap_{i=1}^n \set{s \in S: a(g_i, s) = r_i}
\end{align*}
If there are no guesses or responses so far, $C = S$.
\end{definition}

Candidates are elements of $S$ that \textit{could} be the secret word according to the information contained in the guesses and responses played so far. At the start of a game, the set of candidates is $S$ as no information about the secret has been communicated to the learner yet.

A \textbf{strategy} is a learner's method of determining what guess to submit next, recalling the goal is to get response $r^*$. We make use of candidate sets in formalising this notion.

\begin{definition}[Strategy]\label{def:strategy}
A strategy $\sigma$ can be defined as
\begin{align*}
    \sigma: P(S) \to G
\end{align*}
where $P(S)$ is the power set of $S$. Then $\sigma(C)$, for some candidate set $C \subseteq S$, would represent what guess to submit next.
\end{definition}

Note that the co-domain of $\sigma$ is $G$; the learner is allowed to make guesses that may not be possible secrets. Strategies also may be non-deterministic, but in this paper, we only consider ones that are deterministic.

On receiving a response from the teacher, we need to filter the candidate set appropriately. After submitting a guess, we can know the possible future candidate sets; in this sense each guess can \textbf{split} the candidate set.

\begin{definition}[Split]\label{def:split}
For a candidate set $C$, we say that guessing $g$ creates \textbf{splits} categorized by response $r$:
\begin{align*}
    C_{g,r} = \set{c \in C: a(g,c) = r}
\end{align*}
\end{definition}

We can calculate our score playing according to any strategy using \TurnsNeeded.
\begin{definition}[$\TurnsNeeded$]\label{def:turns_needed}
Suppose we play using strategy $\sigma$ and the hidden secret is $s$. Start with candidate set $C = S$ and submit guess $g = \sigma(C)$. If response $r^*$ is received, then we are done. Otherwise, replace $C$ with $C_{g,r}$ and again submit guess $\sigma(C)$. Repeat until response $r^*$ is received. $\TurnsNeeded(\sigma, s)$ is the number of guesses submitted.
\end{definition}
Note that in this process, the player does not use $s$ to decide what to guess; we only use the remaining candidates and $\sigma$ to determine what to guess next.

The objective of this paper is to find an `optimal' strategy. Existing papers measured the performance of strategies by taking the \Expected{} number of guesses needed (taken over all secrets in $S$) \cite{koyamalai,focardi2012guessing}. Other authors such as \mycite{yams} also considered the maximum number of guesses needed, but the primary goal historically has always been to minimize the \Expected{} score. In this paper we will be using an equivalent metric \Total{} as defined below.

\begin{definition}[\Total{} metric]\label{def:total}
For a strategy $\sigma$, \Total{}$(\sigma)$ is the total number of turns needed over all secrets in $S$:
\begin{align*}
    \Total{}(\sigma) = \sum_{s \in S}\TurnsNeeded(\sigma, s) 
\end{align*}
\end{definition}

It should be clear that a strategy that is optimal according to the \Expected{} case is also optimal according to the \Total~(${=}\,|S|\x\Expected$) case. We use \Total{} however because it makes the work in Section~\ref{sec:proving} much easier to read. 

\subsection{Known Strategies}\label{sec:known_strats}
In the extensive literature on Mastermind \cite{knuth,diag,yams,berghman2009efficient}, several strategies have been developed and tested. We restate some of these strategies in this section for later reference. They all determine what guess to submit next by assigning each guess~$g$ a numerical score based on the current candidate set using a \textbf{valuation}.

\begin{definition}[Valuation-based strategy]\label{def:val:strat}
\begin{align*}
    \sigma_{v}(C) = \argmin{g \in G} \ v(g, C)
\end{align*}
where $C$ is a candidate set and $v$ is some function of type ${G \times P(S) \to \mathbb{R}}$. We call $v$ a \textbf{valuation}. In tie-breaks, default to lexicographical ordering.
\end{definition}
A simple (yet useful) valuation is the following:
\begin{align*}
    \InSet(g, C) = -\mathbb{I}[g \in C]
\end{align*}
where $\mathbb{I}$ is the indicator function. This is adapted from one of the earliest published algorithms on Mastermind \cite{nextvalid}.  We use the negative sign since we are taking the $\min$ in Definition~\ref{def:val:strat}, and prioritising guesses that \textit{are} in $C$.

Strategies developed for Mastermind focused on using different valuations such as:
\begin{align*}
    \MaxSizeSplit{}(g, C) &= \max_{r \in R} \abs{C_{g,r}}\\
    \ExpSizeSplit{}(g,C) &= \sum_{r \in R} \left(\frac{\abs{C_{g,r}}}{\abs{C}} \x \abs{C_{g,r}}\right)\\
    \Information{}(g, C) &= \sum_{r \in R} \frac{\abs{C_{g,r}}}{\abs{C}} \log_2 \frac{\abs{C_{g,r}}}{\abs{C}}\\
    \MostParts{}(g,C) &= -\nSplits{g, C}
\end{align*}
where
\begin{align*}
    \nSplits{g, C} = \abs{\set{r \in R: \abs{C_{g,r}} \neq 0}}
\end{align*}
\mycite{knuth} used \MaxSizeSplit{}, \mycite{yams} introduced \MostParts{} and \mycite{diag} used \Information{}.
 There are several other valuations developed for Mastermind, but we only included the ones which showed promising results in existing literature.
\section{Finding Good Strategies}\label{ch:finding_good_strats}
Before we can prove the optimality of a strategy for any given guessing game, it is necessary to first find a \textit{good} one. We do this by first using \citeauthor{knuth}'s~[\citeyear{knuth}] paper on Mastermind for inspiration, and revisit Definition~\ref{def:total} to develop a method by which we may search for strategies with low \textsc{Total} scores.



\subsection{Combining Known Valuations}\label{subsec:combining_valuations}
\mycite{knuth} used the \MaxSizeSplit{} evaluation. In the event several guesses had equally minimal valuations, he suggested that for the next guess, ``a valid one should be used'', i.e.\ a guess that is also a candidate. He made no explicit rule about which to choose if there are multiple guesses with equally minimal valuations \textbf{and} that are candidates. We resolve this in the context of a general guessing game.

\begin{definition}[Combined valuations]\label{def:val:combined}
For valuations $v_1,\dots,v_n$, we can combine them to assign each guess $g$ a tuple of values:
\begin{align*}
    V(g,C) = \tuple{v_1(g,C), \dots, v_n(g,C)}
\end{align*}
We can then compare tuples lexicographically.
\end{definition}
Only if the combined valuations together are the same for two guesses, \textit{then} we may revert to choosing alphabetically, but ideally we would append more valuations to avoid this.

In choosing which valuations to combine for Wordle, we tested\footnote{Full source code for this experiment and all subsequent ones is available at \hyperlink{https://github.com/cunananm2000/WordleBot}{https://github.com/cunananm2000/WordleBot}.} every non-empty ordered combination of \InSet{}, \MaxSizeSplit{}, \Information{}, \MostParts{}, \ExpSizeSplit{}, giving 325 combined valuations.
\begin{table}
    \centering
    \begin{tabular}{llr}
    \toprule
    \textbf{Rank} &                                \textbf{Combined Valuations} &   \Total{}\\
    \midrule
    1   &                       \MostParts{},\InSet{},\textsc{ESS} &  7944\\
    2   &          \MostParts{},\InSet{},\textsc{ESS},\textsc{MSS} &  7944\\
    3   &          \MostParts{},\InSet{},\textsc{MSS},\textsc{ESS} &  7944\\
    \vdots & \vdots & \vdots\\
    323 &                                       \textsc{MSS} &  8510\\
    324 &                                 \InSet{},\textsc{MSS} &  8516\\
    325 &                                              \InSet{} &  10069\\
    \bottomrule
    \end{tabular}
    \caption{Combined valuations on Wordle. For space, we shorten \ExpSizeSplit{} to \textsc{ESS} and \MaxSizeSplit{} to \textsc{MSS}.}
    \label{tab:combined_vals}
\end{table}

As shown by Table~\ref{tab:combined_vals}, using combined valuations does offer an improvement over using any single valuation alone. Moreover, all the best-performing combinations use \MostParts{} as the main valuation, which makes sense as it was the best-performing standalone valuation for Mastermind \cite{yams}.
\subsection{Searching For Strategies}\label{sec:searching_for_strats}
\citeauthor{minguesses}~[\citeyear{minguesses}] presented an equation to calculate the minimum \Expected{} score achievable.
We adapt this to an equation to calculate the minimum \Total{} score:
\begin{definition}[$\textsc{MinTotal}$]\label{def:min_total}
For a non-empty candidate set $C$, the minimum \Total{} number of guesses needed to reach all candidates in $C$ is given by
\begin{align*}
    \textsc{MinTotal}(C) = |C| + \min_{g \in G} \sum_{r \in R\backslash\set{r^*}}\textsc{MinTotal}(C_{g,r})
\end{align*}
If $C$ is empty, then $\textsc{MinTotal}(C) = 0$.
\end{definition}
Unfortunately, it isn't feasible to calculate \textsc{MinTotal} for most candidate sets in real guessing games such as Wordle. The recursive definition means that with each calculation of \textsc{MinTotal} we loop over $g \in G$.  If we were to limit our recursion depth to $d$, then our algorithm would run in $O(|G|^d)$. To limit this exponential growth, we propose that instead of searching over all $g \in G$, only search over the `best' $n$ guesses in $G$; we call $n$ the \textbf{search breadth}.

\begin{definition}[Approximate \textsc{MinTotal}]\label{def:approx_mintotal}
For a candidate set $C$, the \textbf{approximate} minimum total number of guesses needed is $\textsc{ApMinTotal}(C, n)$, defined by replacing $g \in G$ in Definition~\ref{def:min_total} with $g \in \set{\text{Best } n \text{ guesses}}$.
\end{definition}
We will be using the topmost combined valuation from Table~\ref{tab:combined_vals}, $(\MostParts{}, \InSet{}, \ExpSizeSplit{})$, to determine what the best $n$ guesses are; taking the $n$ guesses with the lowest valuations.

Depending on how exhaustively we want to look for strategies we may change $n$; a higher value of $n$ means a more exhaustive search. It should be clear, then, that ${\textsc{ApMinTotal}(C, n) \geq \textsc{MinTotal}(C)}$ for any $n\geq 1$, and that at $n = \abs{G}$ the two are equal.

As mentioned previously, we may use the $\mathrm{argmin}$ of the `otherwise' case of Definition~\ref{def:approx_mintotal} to extract a strategy. 

Since Wordle is the main guessing game for this paper, we first show our results in Table~\ref{tab:apmintotal:wordle}. As we expect, the \Total{} decreases as the search breadth increases, as this means we search more exhaustively.

\begin{table}
    \centering
    \begin{tabular}{ccc}
    \toprule
    $n$ & \textsc{ApMinTotal}$(S_W, n)$ &\textbf{Starter}\\
    \midrule
    1 & 7944 & \code{TRACE}\\
    5 & 7921 & \code{SALET}\\
    10 & 7920 & \code{SALET}\\
    20 & 7920 & \code{SALET}\\
    \bottomrule
    \end{tabular}
    \caption{Using \textsc{ApMinTotal} to find a good strategy for Wordle.}
    \label{tab:apmintotal:wordle}
\end{table}

We repeated this process of using Definition~\ref{def:approx_mintotal} to search for strategies with low \Total{} scores on the following variants of Wordle:
\begin{itemize}
    \item \textbf{FFXIVrdle}: ${S = }$ References to the video game Final Fantasy 14, e.g.\ \code{HILDA}.
    \item \textbf{Mininerdle}: ${G = S = }$ 6 character math equations, e.g.\ \code{4*7=28}.
    \item \textbf{Nerdle}: ${G = S = }$ 8 character math equations, e.g.\ \code{8*3+2=26}.
    \item \textbf{Primel}: ${G = S = }$ 5 digit prime numbers, e.g.\ \code{42821}.
\end{itemize}
Results are shown in Table~\ref{tab:wordle_variants}.
\begin{table}
    \centering
    \begin{tabular}{cccc}
    \toprule
    \textbf{Game} & $|G|$ & $|S|$ & \textsc{ApMinTotal}$(S, 20)$\\
    \midrule
    FFXIVrdle & 849 & 168 & 432\\
    Mininerdle & 206 & 206 & 544\\
    Nerdle & 17723 & 17723 & 53512\\
    Primel & 8363 & 8363 & 29011\\
    \bottomrule
    \end{tabular}
    \caption{Using \textsc{ApMinTotal} to find a good strategy for other popular guessing games, for which we purposely show the \Expected{} rather than \Total{}. The choice of breadth 20 was due to Table~\ref{tab:apmintotal:wordle}.}
    \label{tab:wordle_variants}
\end{table}

\section{Proving Optimality}\label{sec:proving}
The previous section was focused on using heuristics to find good strategies; now we'd like to determine if the best ones found were indeed optimal. First, we revisit Definition~\ref{def:min_total}, and explore the idea of representing strategies as `trees'. Doing so allows us to prove several propositions which we use to create novel theorems by which we can prove a strategy optimal \emph{without exhaustive search}.

\subsection{Useful Guesses}\label{subsec:ug}
In order to help restrict the search space, we define the notion of \textit{usefulness}.
\begin{definition}\label{def:ug}
For a candidate set $C$ with $\abs{C} > 1$, a guess ${g \in G}$ is \textbf{useful} w.r.t.\ $C$ iff $\nSplits{g,C} \neq 1$.
We notate this as 
$g \in \UG{C}$ for short.
If $\abs{C} \leq 1$, then $\UG{C} = C$.
\end{definition}
\begin{property}\label{property:ug}
Equivalently for ${|C| > 1}$, ${g \in \UG{C}}$ iff $\abs{C_{g,r}} < \abs{C}$ for all $r \in R$.
\end{property}
\begin{lemma}\label{lemma:useful_best}
For any non-empty $C \subseteq S$,
\begin{align*}
\MinTotal{C} &= \!\!\min_{g \in \UG{C}} \abs{C} + \!\!\!\sum_{r \in R\backslash\set{r^*}} \MinTotal{C_{g,r}}
\end{align*}
Note the replacement of ${g \in G}$ from Definition~\ref{def:min_total} with ${g \in \UG{C}}$.
\end{lemma}
\begin{proof}
Suppose the minimum was achieved by some $g \not\in \UG{C}$, so $C_{g,r} = C$ for a specific $r$ and $C_{g,r'} = \emptyset$ for $r'\in R\setminus\{r\}$. By Definition~\ref{def:min_total} this would imply $\MinTotal{C} = |C| + \MinTotal{C}$, so $|C|=0$, contrary to the assumption. 
\end{proof}
This also shows that optimal strategies can only have useful guesses.

\subsection{Setup}\label{subsec:setup}

\begin{definition}[$V^*$]\label{def:mintotal_v*}
By splitting Lemma~\ref{lemma:useful_best}, we define
\begin{align*}
    \MinTotal{C} &= \min_{g \in \UG{C}} V^*(g, C)\\
    V^*(g, C) &= \abs{C} + \sum_{r \in R\backslash\set{r^*}} \MinTotal{C_{g,r}}
\end{align*}
and $\MinTotal{\emptyset} = 0$.
\end{definition}
$V^*(g,C)$ then represents the minimum \Total{}, starting from candidate set $C$, \textbf{provided} we guess ${g \in G}$ first. The \textbf{optimal} strategy would then be achieved by taking the $\mathrm{argmin}$. 

$\MinTotal{C}$ and $V^*(g,C)$ are what we should try to estimate. Finding an upper bound for $\MinTotal{C}$ is easy; as noted previously $\textsc{ApMinTotal}(C,n)$ is an upper bound for any $n$. We notate such an upper bound as $\textsc{UB}(C)$. As per Table~\ref{tab:apmintotal:wordle}, the lowest known value found for Wordle, $\textsc{UB}(S_W)$, is 7920. Lower bounding $V^*$ is important via the following theorem:
\begin{theorem}\label{th:v*_estimate}
Suppose we have some function $\textsc{UB}$ such that $\textsc{UB}(C) \geq \MinTotal{C}$ for any $C \subseteq S$. If we can find some estimate function $V'$ such that $V'(g,C) \leq V^*(g,C)$ for any guess ${g \in G}$ and $C \subseteq S$, then for any $g' \in G$
\begin{align*}
    V'(g',C) > \textsc{UB}(C) \implies g' \neq \argmin{g \in \UG{C}}\, V^*(g, C)
\end{align*}
\end{theorem}
\begin{proof}
If $V'(g',C) > \textsc{UB}(C)$ for a particular $g' \in G$,
\begin{align*}
    V^*(g',C) \geq V'(g',C) > \textsc{UB}(C) \geq \MinTotal{C}
\end{align*}
We then use Definition~\ref{def:mintotal_v*} to complete the proof.
\end{proof}
Theorem~\ref{th:v*_estimate} has the effect that for any guess~$g'$, if ${V^*(g', S) > \textsc{UB}(S)}$, then $g'$ \textbf{cannot} be an optimal starting word.

In order to estimate $V^*$, we must first estimate $\textsc{MinTotal}$ as it is much easier to create bounds for.

\subsection{Tree Representations}\label{subsec:tree}

\mycite{ville} demonstrated representing their Mastermind strategy as a decision tree, and we may do the same with strategies in guessing games in general, as illustrated in Figure~\ref{fig:wordle_tree} for a Wordle strategy. We will call these \textbf{strategy trees}. Each node is a guess to be submitted; starting at the root node as the initial guess. The outgoing branches from a node represent the possible responses received by submitting the node's guess. If $r^*$ is a possible response, then we do \textbf{not} include that branch and instead highlight the node in green as a possible end to the game.

\begin{figure}
    \centering
    \includegraphics[width=\linewidth]{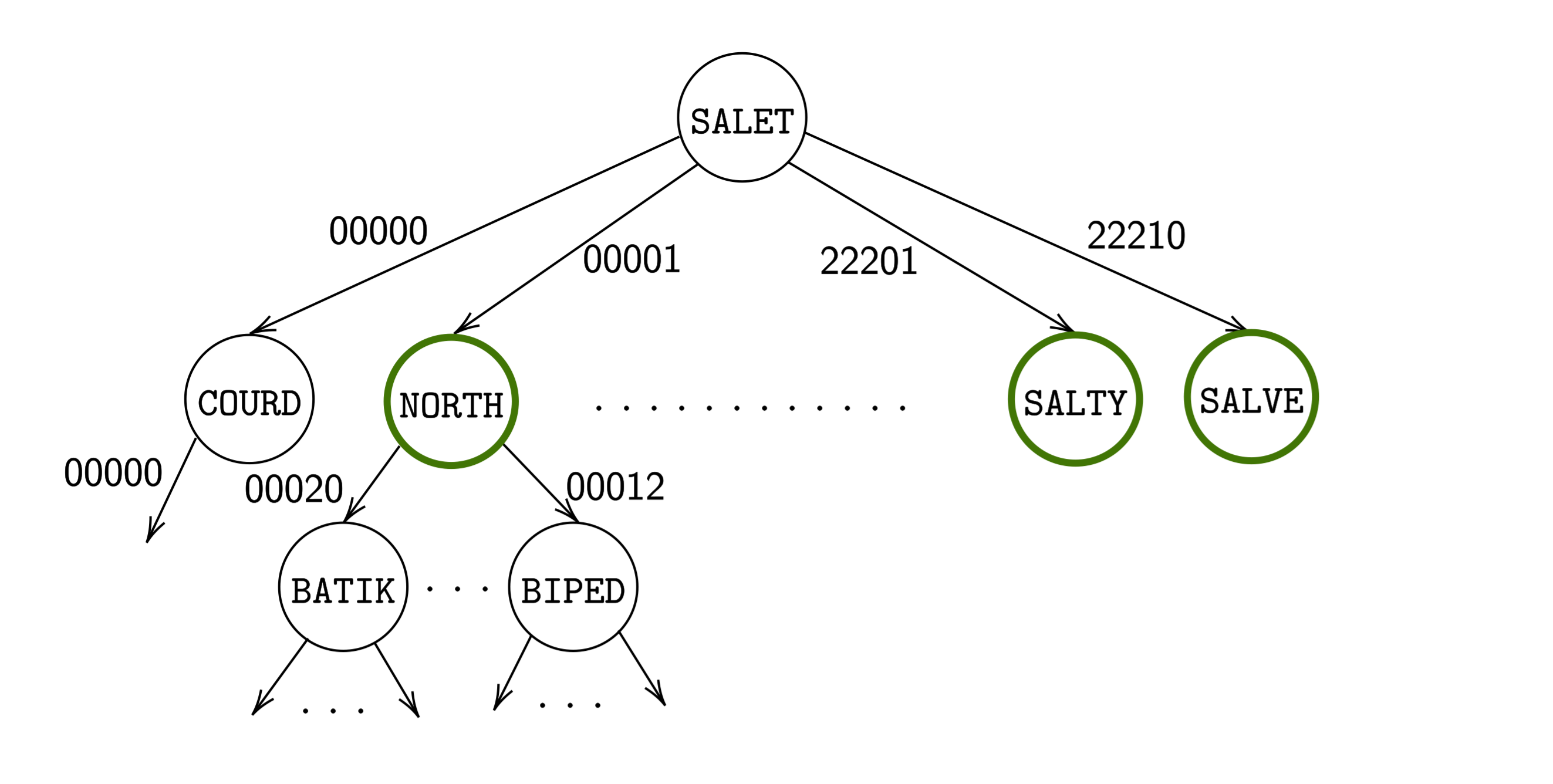}
    \caption{Part of our Wordle strategy represented as a strategy tree. Not all branches or guesses are shown.}
    \label{fig:wordle_tree}
\end{figure}

Alternatively, each node can be thought of as corresponding to a current set of candidates $C$, labelled with $\sigma(C)$.

It follows that for the tree representation of a strategy $\sigma$, the value $\textsc{TurnsNeeded}(\sigma, s)$ is represented by the depth of node labelled with $s$ \emph{as a leaf node}. Note that the same word may appear multiple times in a strategy tree, so we must follow the nodes and branches to properly compute the `correct' depth. Importantly, we assign the depth of the root node as 1, so the value $\Total{}(\sigma)$ can be then visualized as the sum of the depths of each secret. With this, we can then notice that \textsc{MinTotal} is purely dependent on the placements of the nodes corresponding to possible secrets within a tree. The natural question to ask then is, ``What is the best way to arrange the nodes corresponding to possible secrets in a strategy tree to minimize the sum of depths to each of these nodes?'', or put more generally,
\begin{center}
    ``What is the best way to arrange the $n$ nodes in a tree to minimize the sum of depths to each node?''
\end{center}
To answer this we need to prove some properties about strategy trees.

\begin{definition}\label{def:maxsplits}
For any $C \subseteq S$,
\begin{align*}
    \MaxSplits{C} = \max_{g \in G} \nSplits{g, C}
\end{align*}
\end{definition}

\begin{lemma}\label{lemma:maxsplits}
For candidate sets $C'$ and $C$, if ${C' \subseteq C}$, then ${\MaxSplits{C'} \leq \MaxSplits{C}}$.
\end{lemma}
\begin{proof}
If ${C' \subseteq C}$, for any guess ${g \in G}$ we must have $\nSplits{g, C'} \leq \nSplits{g, C}$. This is because if $C'_{g,r}$ is non-empty for some response $r$, then $C_{g,r}$ is non-empty. This implies the desired result.
\end{proof}

\begin{theorem}\label{th:max_splits_tree}
In any tree made to resolve a candidate set $C$, all nodes have at most $\textsc{MaxSplits}(C)$ children.
\end{theorem}

\begin{proof}
The value $\textsc{MaxSplits}(C)$ is the highest number of branches the root node can have. This is true even after noting that $r^*$ is never assigned a corresponding branch. Recall moreover that each child node also corresponds to a candidate set $C' \subset C$. The number of children that the direct child nodes of the root node can have is upper bounded by $\textsc{MaxSplits}(C')$, but by Lemma~\ref{lemma:maxsplits}, this value is upper bounded by $\textsc{MaxSplits}(C)$. The same logic can be cascaded down each branch of the tree to show that each node has at most $\textsc{MaxSplits}(C)$ children.
\end{proof}

\begin{definition}\label{def:bound}
For integers $n \geq 0$ and $b \geq 1$, $\Bound{n,b}$ is the minimum sum of depths of each node in a tree with $n$ nodes, and each node having at most $b$ children. We call such a tree a $b$-tree.
\end{definition}

\begin{theorem}\label{th:bound}
For integers $n > 0$ and $b > 1$,
\begin{align*}
    &\Bound{n, b} = \sum_{i=1}^k ib^{i-1} + (k + 1)\x \left(n - \frac{b^k-1}{b-1}\right) \\
    &\text{ where } k = \floor{\log_b(n(b-1)+1)} 
\end{align*}
For $b = 1$, $\textsc{Bound}(n, 1) = \frac{n(n+1)}{2}$.\\
For $n = 0$, $\textsc{Bound}(0, b) = 0$.
\end{theorem}
\begin{proof}
Note that $\textsc{Bound}(0,1) = 0$ by either of the last two cases.

This is trivial when $n = 0$ or $b = 1$. In the case where $n > 0$ and $b > 1$, there are multiple ways to arrange $n$ nodes in a $b$-tree. We are interested in minimizing the sum of depths to each node; clearly we must fill in level-order, noting that there are at most $b^{i-1}$ nodes of depth $i$. Doing this will show that $k$ is the depth of the last completely filled layer. The last term in the definition accounts for the `leftover' nodes at depth $k+1$.
\end{proof}

The restriction on the number of children suggests we use $\textsc{Bound}$ in creating lower bounds for $\textsc{MinTotal}$.
\begin{definition}\label{def:lb1_2}
\begin{align*}
    LB_1(C) &= \Bound{|C|, \MaxSplits{S}}\\
    LB_2(C) &= \Bound{|C|, \MaxSplits{C}}
\end{align*}
\end{definition}
In turn, we use this to recursively build bounds for $V^*$ and $\textsc{MinTotal}$, taking inspiration from Definition~\ref{def:mintotal_v*}.
\begin{definition}\label{def:build}
For integers $i \geq 1$
\begin{align*}
    LB_{i+2}(C) &= \min_{g \in \UG{C}} V_i(g, C)\\
    V_i(g,C) &= \abs{C} + \sum_{r \in R\backslash\set{r^*}} LB_i(C_{g,r})
\end{align*}
\end{definition}
It remains to prove that $LB_i$ and $V_i$ are in fact lower bounds to $\textsc{MinTotal}$ and $V^*$ respectively.
\subsection{Key Theorems and Proofs}
We now use the work of Sections~\ref{subsec:ug},~\ref{subsec:setup} and \ref{subsec:tree} to build the key theorems of this paper.
\begin{lemma}\label{lemma:lb1_2}
    For any $C \subseteq S$, $LB_1(C) \leq LB_2(C)$.
\end{lemma}
\begin{proof}
Note that $\Bound{n,b}$ increases as $b$ decreases for a fixed $n$. Recall the definition of $\Bound{n,b}$. Clearly decreasing $b$ means each node's depth can only increase, so the overall sum of depths for each node must increase.

This fact combined with Lemma~\ref{lemma:maxsplits} implies the desired result.
\end{proof}
\begin{lemma}\label{lemma:lb2_mintotal}
    For any $C \subseteq S$, $LB_2(C) \leq \MinTotal{C}$.
\end{lemma}
\begin{proof}
\textsc{MinTotal} is intended to represent the best way to arrange the candidates of $C$ in any valid strategy tree in order to minimize the sum of depths to each candidate. We know that in this `ideal' strategy tree, there must be at least $|C|$ nodes (one for each candidate), and that by Theorem~\ref{th:max_splits_tree} this tree is a $\textsc{MaxSplits}(C)$-tree. It follows by Definition~\ref{def:lb1_2} and Definition~\ref{def:bound} that the sum of depths to each node is lower bounded by $LB_2(C)$.
\end{proof}
\begin{lemma}\label{lemma:lbi_j_v_i_j}
If $LB_i(C) \leq LB_j(C)$ for any $C \subseteq S$, then $V_i(C) \leq V_j(C)$ for any guess ${g \in G}$ and $C \subseteq S$.
\end{lemma}
\begin{proof}
Follows from construction in Definition~\ref{def:build}.
\end{proof}
\begin{proposition}\label{p:vi_j_lb_i+2_j+2}
If $V_i(C) \leq V_j(C)$ for any guess ${g \in G}$ and $C \subseteq S$, then $LB_{i+2}(C) \leq LB_{j+2}(C)$ for any $C \subseteq S$.
\end{proposition}
\begin{proof}
Follows from construction in Definition~\ref{def:build}.
\end{proof}
\begin{corollary}\label{co:lbi_j_i+2_j+2}
If ${LB_i(C) \leq LB_j(C)}$ for any ${C \subseteq S}$, then ${LB_{i+2}(C) \leq LB_{j+2}(C)}$ for any ${C \subseteq S}$.
\end{corollary}
\begin{proof}
Follows from Lemma~\ref{lemma:lbi_j_v_i_j} and Proposition~\ref{p:vi_j_lb_i+2_j+2}.
\end{proof}
\begin{theorem}\label{th:lb_adjacent}
For any integer ${n \geq 1}$, we have that ${LB_{2n-1}(C) \leq LB_{2n}(C) \leq \MinTotal{C}}$ for any $C \subseteq S$.
\end{theorem}
\begin{proof}
We may follow a similar proof to Corollary~\ref{co:lbi_j_i+2_j+2} to show that if $LB_i(C) \leq \MinTotal{C}$ for any $C \subseteq S$, then $LB_{i+2}(C) \leq \MinTotal{C}$ for any $C \subseteq S$.

The desired result follows combining this with Lemmas~\ref{lemma:lb1_2} and~\ref{lemma:lb2_mintotal} along with Corollary~\ref{co:lbi_j_i+2_j+2}.
\end{proof}

\begin{theorem}\label{th:lb_same}
For any integer ${n \geq 1}$, we have that ${LB_{n}(C) \leq LB_{n+2}(C) \leq \MinTotal{C}}$ for any ${C \subseteq S}$.
\end{theorem}
\begin{hproof}
We can show that for any $C \subseteq S$, we have ${LB_1(C) \leq LB_3(C)}$. Intuitively, $LB_1(C)$ is the minimum sum of depths in a tree assuming that each node has at most $\MaxSplits{S}$ children. $LB_3(C)$ however asserts that the root node \textbf{must} split according to a legitimate guess. This restriction implies $LB_1(C) \leq LB_3(C)$. 

We can also show that ${LB_2(C) \leq LB_4(C)}$, for any ${C \subseteq S}$. Proving this uses the previous claim of ${LB_1(C) \leq LB_3(C)}$. It requires the trick that we may replace $S$ in the explicit definition of that claim with $C$, since we only require that $S$ be a super-set of $C$.

Combine these two inequalities with Lemmas~\ref{lemma:lb1_2} and~\ref{lemma:lb2_mintotal}, and Corollary~\ref{co:lbi_j_i+2_j+2} to reach the desired conclusion.
\end{hproof}
\begin{proposition}\label{p:v_v*}
For any integer $n \geq 1$, $V_n(g, C) \leq V^*(g, C)$ for any guess ${g \in G}$ and any $C \subseteq S$.
\end{proposition}
\begin{proof}
Theorem~\ref{th:lb_adjacent} shows that $LB_n(C) \leq \MinTotal{C}$ for any integer $n \geq 1$ and any $C \subseteq S$.

By the construction of $V^*$ (Definition~\ref{def:mintotal_v*}) and $V_i$ (Definition~\ref{def:build}), we can use the above to conclude $V_n(C) \leq V^*(C)$ for any integer $n \geq 1$ and any $C \subseteq S$.
\end{proof}
\begin{theorem}\label{th:v_same}
For any integer ${n \geq 1}$, we have that ${V_{n}(g,C) \leq V_{n+2}(g, C) \leq V^*(g, C)}$ for any guess ${g \in G}$ and $C \subseteq S$.
\end{theorem}
\begin{proof}
Proposition~\ref{p:v_v*} and Lemma~\ref{lemma:lbi_j_v_i_j} imply that inequalities from Theorem~\ref{th:lb_same} also hold if we replace each $LB$ with $V$, keeping subscripts the same, which was what we wanted.
\end{proof}

This shows we have developed an infinite system of lower-bounds for $V^*$. Recall how we plan to use these as stated in Theorem~\ref{th:v*_estimate}.

\begin{theorem}\label{th:v_ends}
For any guess ${g \in G}$ and any ${C \subseteq S}$, we have ${V_{2|C| + 1}(g, C) = V^*(g,C)}$.
\end{theorem}
\begin{proof}
First we prove a similar statement about $LB$, that $LB_{2|C| + 1}(C) = \MinTotal{C}$ for any ${C \subseteq S}$. We do this by way of induction.

The base case of $|C| = 0$ is trivial. Assume then this is true for any $|C| \leq M$ for some integer $M \geq 0$, and
suppose we have some $C$ where $|C| = M + 1$. Then we have
\begin{align*}
LB_{2|C| + 1}(C) = \min_{g \in \UG{C}} |C| + \sum_{r \in R\backslash\set{r^*}} LB_{2M+1}(C_{g,r})
\end{align*}
Because we are only considering ${g \in \UG{C}}$, we have ${C_{g,r} \subset C}$, implying $\abs{C_{g,r}} \leq M$. Since $2\abs{C_{g,r}} + 1$ and $2M+1$ are odd, Theorem~\ref{th:lb_same} and the induction step imply that $LB_{2M+1}(C_{g,r}) = \MinTotal{C_{g,r}}$. 

Using Definition~\ref{def:build} and Definition~\ref{def:mintotal_v*} from this completes the induction.

Now that we know ${LB_{2|C| + 1}(C) = \MinTotal{C}}$ for any $C \subseteq S$, Definition~\ref{def:build} and Definition~\ref{def:mintotal_v*} again can be used to achieve the desired result.
\end{proof}

\begin{theorem}\label{th:ub}
Suppose we have an upper bound $\textsc{UB}(C)$ for $\textsc{MinTotal}(C)$. If for all ${g \in G}$ there exists an $i_g$ such that $V_{i_g}(g, C) \geq \textsc{UB}(C)$, then $\textsc{UB}(C) = \textsc{MinTotal}(C)$
\end{theorem}
\begin{proof}
By Theorem~\ref{th:lb_adjacent}, we may `round up' any odd $i_g$ to an even $i_g + 1$ and would still have $V_{i_g + 1}(g, C) \geq \textsc{UB}(C)$. Hence w.l.o.g we may assume that all $i_g$ are even. Define \mbox{$I = \max_{g \in G} i_g$}, which must also be even. 
Theorem~\ref{th:lb_same} lets us state that $V_I(g, C) \geq V_{i_g}(g,C)$ for all ${g \in G}$, implying that
\begin{align*}
    \textsc{UB}(C) \leq \min_{g \in G} V_I(g, C) = LB_{I+2}(C) \leq \MinTotal{C}
\end{align*}
The construction of $\textsc{UB}(C)$ implies the desired result.
\end{proof}

These theorems are the basis for how we can determine an optimal starting guess and subsequently determine if a strategy is \textbf{provably} optimal:
\begin{enumerate}
  \item Given a guessing game with secret set $S$, start with an upper bound $\textsc{UB}(S)$, found by Definition~\ref{def:approx_mintotal}. Recall that by Theorem~\ref{th:v*_estimate}, we can use any lower bound for $V^*$ in conjunction with $\textsc{UB}(S)$ to rule out which guesses could not be the starting guess of an optimal strategy. Theorem~\ref{th:v_same} gives us this lower bound for $V^*$.

  \item Rule out any guess $g$ for which ${V_1(g, S) > \textsc{UB}(S)}$, then rule out any for which ${V_2(g, S) > \textsc{UB}(S)}$, and so on. Do this until there is one guess left, or ${V_n(g, S_W) \geq \textsc{UB}(S)}$ for all the remaining guesses (in which case we have multiple optimal starting guesses, as shown by Theorem~\ref{th:ub}), or up until calculating $V_{2|S|+1}$ (by Theorem~\ref{th:v_ends}).
\end{enumerate}

\section{Application}
We demonstrate using the general method and theorems of Section~\ref{sec:proving} on Wordle to test if our best Wordle strategy found in Section~\ref{sec:searching_for_strats} by Definition~\ref{def:approx_mintotal} is optimal.

In Table~\ref{tab:apmintotal:wordle}, the best \Total{} for Wordle we found was $7920$, so $\textsc{MinTotal}(S_W) \leq 7920$ where $S_W$ is the secret set of Wordle. We then applied the method summarized at the end of the previous section to determine the optimal starting word, as well as the true value of $\textsc{MinTotal}(S_W)$.

We present our results in Table~\ref{tab:v_filter}. After applying $V_5$, only one guess remained, \code{SALET}. This agrees with our strategy found in Section~\ref{ch:finding_good_strats}, in which the starting guess was indeed \code{SALET}. Although $V_5(\code{SALET}, S_W) = 7919$, one further iteration showed that $V_6(\code{SALET}, S_W) = 7920$, so not only did we find the optimal starting word, but the strategy found in Section~\ref{ch:finding_good_strats} is a \emph{provably} optimal strategy for minimizing the \Total{}.

\begin{table}
    \centering
    \begin{tabular}{ccc}
        \toprule
        $i$ & \textbf{After filtering by $V_i$} & $\min_{g \in G}V_i(g)$\\
        \midrule
        1 & 12453 & 6829\\
        2 & 1711 & 7664\\
        3 & 324 & 7795\\
        4 & 138 & 7826\\
        5 & 1 & 7919\\
        6 & 1 & 7920\\
        \bottomrule
    \end{tabular}
    \caption{Using $V_1,\dots,V_5$ to filter potential starting Wordle guesses, starting with 12972 possible guesses.}
    \label{tab:v_filter}
\end{table}

This process of filtering by $V_i$ to determine $\textsc{MinTotal}$ was repeated for FFXIVrdle and Mininerdle, all having shown that the strategy found was optimal.
\section{Conclusion} \label{sec:conclusion}
This paper produced two main contributions. First, we used combined valuations to leverage information in determining good strategies for guessing games. Second, we presented several theorems that led to a general theory for mathematically proving a certain strategy optimal, thereby avoiding a complete and exhaustive search. As stated in the introduction, the concrete results produced in this paper were focused on Wordle, but the theory and methodology apply to any game that fits the definition of a general guessing game.

We further hope that these theorems can help in applications of guessing games~\cite{focardi2012guessing} as well as add mathematical rigour to studying optimal context representations in the field of meta-reasoning~\cite{filman1983metalanguage}. Our results could also assist with developing predicates for practical guessing games \cite{thomason2016learning}, or possibly help an AI to learn such predicates. Our theorems could be adapted to enable an AI to determine, out of a set of possible predicates, which are the most `discriminatory'.


In terms of future work, we would like to see this work expanded to guessing games in much looser restrictions, for example in situations where the answering function is non-deterministic. We would also like to find estimating functions that converge to the true answer in fewer iterations, and are faster to compute.

\section*{Acknowledgments}

We thank Abdallah Saffidine for his suggestions and input throughout the duration of this work.


\bibliographystyle{named}
\bibliography{ms}

\newpage
\appendix
\section{Proofs}

For completeness, we include the full proofs of all theorems, properties, propositions and lemmas stated in the paper, keeping consistent with established notation. 

\subsection{Proof of Property~\ref{property:ug}}\label{subsec:prop1}
\begin{proof}
Let $\abs{C} > 1$. By Definition~\ref{def:split} we have $C_{g,r} \subseteq C$ for all $r \in R$, and so $\abs{C_{g,r}} \leq \abs{C}$ for all $r \in R$, with equality clearly iff $C_{g,r} = C$.\\
\\
Note that for any two responses $r_1, r_2 \in R$, and any guess $g \in G$:
\begin{align*}
    C_{g,r_1} \cap C_{g,r_2} \neq \emptyset &\implies \exists c \in C: c \in C_{g,r_1} \land c \in C_{g,r_2}\\
    &\implies \left( a(g,c) = r_1 \right) \land \left( a(g,c) = r_2 \right) \\
    &\implies r_1 = r_2
\end{align*}
and so by contrapositive,
\begin{align*}
    r_1 \neq r_2 \implies C_{g,r_1} \cap C_{g,r_2} = \emptyset.
\end{align*}
By definition of $a$ in Definition~\ref{def:guessing_game}, we must have $a(g,c) \in R$ for any $g \in G$ and $c \in C$. It must be true then for any $g \in G$ that
\begin{align*}
    c \in C &\implies \exists r \in R \ : \ a(g,c) = r \\
    &\implies \exists r \in R \ : c \in C_{g,r}\\
    &\implies c \in \bigcup_{r \in R} C_{g,r}
\end{align*}
so we have $C \subseteq \bigcup_{r \in R} C_{g,r}$. Clearly as well $\bigcup_{r \in R} C_{g,r} \subseteq C$, and so $C = \bigcup_{r \in R} C_{g,r}$.\\
\\
By the inclusion-exclusion principle we can say that
\begin{align*}
    \abs{C} &= \abs{\bigcup_{r \in R} C_{g,r}}\\
    &= \sum_{\emptyset \neq J \subseteq R} (-1)^{|J| + 1} \abs{\bigcap_{r \in J} C_{g,r}}\\
    &= \sum_{r \in R} \abs{C_{g, r}}
\end{align*}
where the last equality holds because any intersection of two different splits is empty.\\
\\
If $g \in \UG{C}$, then at least two splits are non-empty; let these be $C_{g,r_1}$ and $C_{g,r_2}$.
\begin{align*}
    \abs{C_{g,r_1}} &= \abs{C} - \abs{C_{g,r_2}} - \sum_{r \in R \backslash \set{r_1, r_2}} \abs{C_{g,r}}\\
    &\leq \abs{C} - \abs{C_{g,r_2}}\\
    &< \abs{C} \quad \quad \quad \quad \text{(since $C_{g,r_2}$ is non-empty)}
\end{align*}
Similar working shows that ${\abs{C_{g,r_1}} < \abs{C}}$. If $\abs{R} = 2$ then we are done. Otherwise, for any ${r' \not\in \set{r_1, r_2}}$ we have
\begin{align*}
    \abs{C_{g,r'}} &= \abs{C} - \abs{C_{g,r_1}} - \sum_{r \in R \backslash \set{r', r_1}} \abs{C_{g,r}}\\
    &\leq \abs{C} - \abs{C_{g,r_1}}\\
    &< \abs{C}
\end{align*}
So we have $g \in \UG{C} \implies \abs{C_{g,r}} < \abs{C}$ for all $r \in R$.\\
\\
Note that because $\abs{C} > 1$, the result of the inclusion-exclusion principle implies that at least one split is non-empty, i.e.
\begin{align*}
    \abs{C} > 1 \implies \textsc{nSplits}(g,C) \geq 1
\end{align*}
If $\textsc{nSplits}(g,C) = 1$, then there exists $r' \in R$ such that $\abs{C_{g,r'}} = C$ and then $C_{g,r} = \emptyset$ for any $r \in R \backslash \set{r'}$. The contrapositive of this statement then gives
\begin{align*}
    \forall r \in R: \abs{C_{g,r}} < \abs{C} &\implies \textsc{nSplits}(g,C) = 1\\
    &\implies g \in \UG{C}
\end{align*}
We have proven then that Property~\ref{property:ug} is equivalent to Definition~\ref{def:ug}.
\end{proof}

\subsection{Proof of Lemma~\ref{lemma:useful_best}}
\begin{proof}
Define
\begin{align*}
    g' = \argmin{g \in G} \left(\abs{C} + \!\!\!\sum_{r \in R\backslash\set{r^*}} \MinTotal{C_{g,r}}\right)
\end{align*}
Suppose that $g' \not\in \UG{C}$. By Definition~\ref{def:ug} we would know then there exists a unique $r'$ such that
\begin{align*}
    C_{g',r'} = C \text{ and } \forall r \in R\backslash\set{r'}:  C_{g',r} = \emptyset
\end{align*}
We can use these to expand the $\MinTotal{C}$ using Definition~\ref{def:min_total}.
\begin{align*}
    \MinTotal{C} &= \min_{g \in G} \abs{C} + \sum_{r \in R\backslash\set{r^*}} \MinTotal{C_{g,r}}\\
    &= \abs{C} + \sum_{r \in R\backslash\set{r^*}} \MinTotal{C_{g',r}}\\
    &= \abs{C} + \MinTotal{C_{g',r'}}\\
    &= \abs{C} + \MinTotal{C}\\
\end{align*}
where we use the fact that $\MinTotal{\emptyset} = 0$. This however implies that $\abs{C} = 0$, which contradicts the non-emptiness of $C$, so we must have $g' \in \UG{C}$.
\end{proof}


\subsection{Proof of Theorem~1}
\begin{proof}
Let $g'$ be an arbitrary guess in $G$, and $C$ be a subset of $S$. If $V'(g',C) > \textsc{UB}(C)$  then
\begin{align*}
    V^*(g',C) \geq V'(g',C) > \textsc{UB}(C) \geq \MinTotal{C}.
\end{align*}
From Definition~\ref{def:mintotal_v*}, we have
\begin{align*}
    V^*(g', C) > \MinTotal{C} = \min_{g \in \UG{C}} V^*(g,C) 
\end{align*}
and so $g'$ cannot be the $\argmin{}$ of $V^*(g, C)$ over all ${g \in \UG{C}}$.
\end{proof}

\subsection{Proof of Lemma~\ref{lemma:maxsplits}}
\begin{proof}
Let $C'$ and $C$ be candidate sets such that $C' \subseteq C$. For any $r \in R$,
\begin{align*}
    C'_{g,r} \neq \emptyset &\implies \exists c \in C': \ a(g,c) = r\\
    &\implies \exists c \in C: \ a(g,c) = r\\
    &\implies \exists C_{g,r} \neq \emptyset
\end{align*}
Hence
\begin{align*}
    \textsc{nSplits}(g,C') &= \abs{\set{r \in R: \ \abs{C'_{g,r}} \neq 0}}\\
    &\leq \abs{\set{r \in R: \ \abs{C_{g,r}} \neq 0}}\\
    &= \textsc{nSplits}(g,C)
\end{align*}
\end{proof}

\subsection{Proof of Theorem~\ref{th:max_splits_tree}}
\begin{proof}
The value $\textsc{MaxSplits}(C)$ is the highest number of branches the root node can have. This is true even after noting that $r^*$ is never assigned a corresponding branch. Recall also that each child node also corresponds to a candidate set $C' \subset C$. The number of children that the direct child nodes of the root node can have is upper bounded by $\textsc{MaxSplits}(C')$, but by Lemma~\ref{lemma:maxsplits}, this value is upper bounded by $\textsc{MaxSplits}(C)$. The same logic can be cascaded down each branch of the tree to show that each node has at most $\textsc{MaxSplits}(C)$ children.
\end{proof}


\subsection{Proof of Theorem~\ref{th:bound}}
\begin{proof}
Note that $\textsc{Bound}(0,1) = 0$ by either of the last two cases.\\
\\
Clearly, if $n = 0$, then the tree is empty and so the sum of depths is 0.\\
\\
If $b = 1$, then each node has at most 1 child, making the tree equivalent to a linked list. This is clearly the only configuration of a 1-tree with $n$-nodes, and sum of depths would be $1 + \dots + n$, equivalent to $\textsc{Bound}(n, 1)$.\\
\\
In the case where $n > 0$ and $b > 1$, there are multiple ways of arranging $n$ nodes into a tree where each node has at most $b$ children; we are only interested in the configuration the minimizes the sum of depths to each node. Clearly this would be achieved by starting with an empty tree and inserting nodes in a level-order. The first level (depth 1) can have at most 1 node (the root node), the second level can have most $b$ nodes, and it should be clear that depth $i$ can have most $b^{i-1}$ nodes.\\
\\
Let $k$ be the depth of the last completely filled level. The number of nodes that can fit into levels $1,\dots,k$ is
\begin{align*}
    1 + b + \dots + b^{k-1} = \frac{b^k-1}{b-1}
\end{align*}
We need the maximum value of $k$ such that $\frac{b^k-1}{b-1} \leq n$, and so solving for $k$:
\begin{align*}
    \frac{b^k-1}{b-1} &\leq n\\
    b^k - 1 &\leq n(b-1)\\
    b^k &\leq n(b-1) + 1\\
    k &\leq \log_b \left(n(b-1) + 1\right)
\end{align*}
Rounding down the right hand side of the last inequality gives us the same $k$ stated in the definition.\\
\\
This leaves $\left(n - \frac{b^k-1}{b-1}\right)$ nodes at depth $k+1$. There are $b^{i-1}$ nodes at depth $i$ from depths $1$ to $k$, giving us the following total
\begin{align*}
\sum_{i=1}^k ib^{i-1} + (k + 1)\x \left(n - \frac{b^k-1}{b-1}\right)
\end{align*}
\end{proof}


\subsection{Proof of Lemma~\ref{lemma:lb1_2}}
\begin{proof}
Let $C \subset S$. This is trivial if $C = \emptyset$.  
\\
Suppose $C \neq \emptyset$. Note that for any integers $b, b'$ such that $1 \leq b' \leq b$, $\textsc{Bound}(\abs{C}, b) \geq \textsc{Bound}(\abs{C}, b')$. This follows from Definition~\ref{def:bound}; restricting the number of children each node can have while keeping the number of nodes the same can only increase the sum of depths to each node.\\
\\
Since candidate sets are, by definition, subsets of $S$, Theorem~\ref{th:max_splits_tree} implies $\MaxSplits{C} \leq \MaxSplits{S}$. Because $C \neq \emptyset$, $\MaxSplits{C} \geq 1$. Setting $b' = \MaxSplits{C}$ and $b = \MaxSplits{S}$ gives us $\textsc{Bound}(\abs{C}, b) \leq \textsc{Bound}(\abs{C}, b')$, which is what we wanted.
\end{proof}

\subsection{Proof of Lemma~\ref{lemma:lb2_mintotal}}

\begin{proof}
$\textsc{MinTotal}(C)$ is intended to represent the best way to arrange the candidates of $C$ in any valid strategy tree in order to minimize the sum of depths to each candidate. We know that in this `ideal' strategy tree, there must be at least $|C|$ nodes (one for each candidate), and that by Theorem~\ref{th:max_splits_tree} this tree is a $\textsc{MaxSplits}(C)$-tree. The sum of depths to each node then is at least $\textsc{Bound}(|C|, \textsc{MaxSplits}(C))$. This is exactly $LB_2(C)$ by Definition~\ref{def:lb1_2}.
\end{proof}

\subsection{Proof of Lemma~\ref{lemma:lbi_j_v_i_j}}
\begin{proof}
\begin{align*}
    &\forall C \subseteq S: \  LB_i(C) \leq LB_j(C)\\
    &\implies \forall g \in G \forall r \in R \forall C \subseteq S: \ LB_i(C_{g,r}) \leq LB_j(C_{g,r})\\
    &\implies \forall g \in G \forall C \subseteq S: \ V_i(g,C) \leq V_j(g,C)
\end{align*}
\end{proof}

\subsection{Proof of Proposition~\ref{p:vi_j_lb_i+2_j+2}}
\begin{proof}
Let $V_i(C) \leq V_j(C)$ for any guess ${g \in G}$ and $C \subseteq S$. Define
\begin{align*}
    g' = \argmin{g \in \UG{C}} V_j(g, C)
\end{align*}
It follows that
\begin{align*}
    LB_{j+2}(C) = V_j(g',C) \geq V_i(g', C) \geq \min_{g \in \UG{C}} V_i(g, C)
\end{align*}
which is what we wanted.
\end{proof}

\subsection{Proof of Corollary~\ref{co:lbi_j_i+2_j+2}}
\begin{proof}
We use Lemma~\ref{lemma:lbi_j_v_i_j} and Proposition~\ref{p:vi_j_lb_i+2_j+2}.
\begin{align*}
    &\forall C \subseteq S: \ LB_i(C) \leq LB_j(C)\\
    &\implies \forall g \in G \forall C \subseteq S: V_i(g, C) \leq V_j(g, C)\\
    &\implies \forall C \subseteq S: \ LB_{i+2}(C) \leq LB_{j+2}(C)
\end{align*}
\end{proof}

\subsection{Proof of Theorem~\ref{th:lb_adjacent}}
\begin{proof}
First we note the following:
\begin{align*}
    &\forall C \subseteq S: \ LB_i(C) \leq \MinTotal{C}\\
    &\implies \forall g \in G \forall r \in R \forall C \subseteq S:\\
    &\quad \quad \quad \quad \quad \ LB_i(C_{g,r}) \leq \MinTotal{C_{g,r}}\\
    &\implies \forall g \in G \forall C \subseteq S: \ V_i(g, C) \leq V^*(g, C)
\end{align*}
Similar working to Proposition~\ref{p:vi_j_lb_i+2_j+2} shows that
\begin{align*}
    &\forall g \in G \forall C \subseteq S: \ V_i(g, C) \leq V^*(g, C)\\
    &\implies LB_{i+2}(C) \leq \MinTotal{C}
\end{align*}
Putting this with earlier working shows that
\begin{align*}
    &\forall C \subseteq S: \ LB_i(C) \leq \MinTotal{C}\\
    &\implies LB_{i+2}(C) \leq \MinTotal{C}.
\end{align*}
We have shown from Lemma~\ref{lemma:lb1_2} and Lemma~\ref{lemma:lb2_mintotal} that for any $C \subseteq S$,
\begin{align*}
    LB_1(C) \leq LB_2(C) \leq \MinTotal{C}
\end{align*}
Earlier working and Corollary~\ref{co:lbi_j_i+2_j+2} can then be used to show that
\begin{align*}
    LB_3(C) \leq LB_4(C) \leq \MinTotal{C}
\end{align*}
or in general, that for any integer $n \geq 1$,
\begin{align*}
    LB_{2n-1}(C) \leq LB_{2n}(C) \leq \MinTotal{C}
\end{align*}
\end{proof}

\subsection{Proof of Theorem~\ref{th:lb_same}}
This is by far the longest proof of the paper. The informal proof gives intuition to the work that follows. We will need to state and prove several new definitions and lemmas to prove Theorem~\ref{th:lb_same}.\\
\\
Note that for a fixed guessing game, $R$ and $S$ are fixed, so both $\abs{R}$ and $\MaxSplits{S}$ are constant. Let $M = \MaxSplits{S}$.

\begin{definition}[\textsc{Bound} as $d$]
Define a new function $d(n)$ as
\begin{align*}
    d(n) = \textsc{Bound}(n, \MaxSplits{S}) = \textsc{Bound}(n, M)
\end{align*}
\end{definition}
This is defined purely for convenience. Notice then we may rewrite Definitions~\ref{def:lb1_2} and \ref{def:build} as
\begin{align*}
    LB_1(C) &= d(\abs{C})\\
    V_1(g, C) &= \abs{C} + \sum_{r \in R\backslash\set{r^*}} d(\abs{C_{g,r}})
\end{align*}

\begin{lemma}\label{lemma:d}
For integers $m \geq 1$ and $n \geq 0$,
\begin{align*}
    m > n \implies d(m) > d(n)
\end{align*}
\end{lemma}
\begin{proof}
By definition of $\textsc{Bound}$, it should be clear that if $m > n$ (i.e.\ more nodes) then the minimum sum of depths in an $M$-tree with $m$ nodes should be strictly greater than the minimum sum of depths in an $M$-tree with $n$ node.
\end{proof}

\begin{corollary}\label{co:d}
For any $n \geq 0$, $d(n+1) - 1 \geq d(n)$
\end{corollary}

\begin{proof}
Note that the co-domain of $d$ is $\mathbb{N}_0$ (the non-negative integers). Starting from Lemma~\ref{lemma:d},
\begin{align*}
    d(n+1) > d(n) &\implies d(n+1) > d(n)\\
    &\implies d(n+1) \geq d(n) + 1\\
    &\implies d(n+1) - 1 \geq d(n)
\end{align*}
\end{proof}

\begin{lemma}\label{lemma:A}
Define $A: \mathbb{N}_0 \to P((\mathbb{N}_0)^M)$ as
\begin{align*}
    A(n) = \set{(a_1,\dots,a_M) \in (\mathbb{N}_0)^M : \sum_{i=1}^M a_i = n-1}
\end{align*}
For any integer $n \geq 0$, we claim
\begin{align*}
    d(n) = n + \min_{(a_1,\dots,a_M) \in A(n)} \sum_{i=1}^M d(a_i)
\end{align*}
\end{lemma}

\begin{proof}
This is induced by an alternative way of constructing the $M$-tree with $n$ nodes with the minimum sum of depths to each node.\\
\\
We clearly must have a root node, and like before this will be at depth $1$. There are then at most $M$ subtrees from this root node. We can instead interpret this as rather the root node having exactly $M$ subtrees, allowing for an empty subtree. The total number of nodes in these subtrees (from the root node) is $n-1$, since one node has already been assigned as the root node.\\
\\
Suppose we assign the $i^{th}$ branch of the root node to~have $a_i \geq 0$ nodes, such that $\sum_{i=1}^M a_i = n - 1$. These $a_i$ nodes should be arranged optimally within their respective subtrees, i.e. each subtree will have total sum of depths~$d(a_i)$. With this assignment of nodes, we can say the sum of depths to each node from the root node is
\begin{align*}
    1 + \sum_{i=1}^M (a_i + d(a_i)) &= 1 + \sum_{i=1}^M a_i + \sum_{i=1}^M d(a_i) \\
    &= n + \sum_{i=1}^M d(a_i)
\end{align*}
The assignment of $a_i$'s can be changed as long as $\sum_{i=1}^M a_i = n-1$. To get the optimal assignment, we simply take the minimum of the above.
\begin{align*}
    d(n) = n + \min_{(a_1,\dots,a_M) \in A(n)} \sum_{i=1}^M d(a_i)
\end{align*}
\end{proof}

\begin{lemma}\label{lemma:c_g_r*}
For any candidate set $C$ and any $g \in G$, $$\abs{C_{g,r^*}} = \mathbb{I}[g \in C]$$ where $\mathbb{I}$ is the indicator function. 
\end{lemma}
\begin{proof}
We solve this in cases.\\
\textbf{\textit{Case 1.}}
\begin{align*}
    g \not \in C &\implies \forall c \in C: \ g \neq c\\
    &\implies \forall c \in C: a(g,c) \neq r^*\\
    &\implies \forall c \in C: c \neq C_{g,r^*}\\
    &\implies C_{g,r^*} = \emptyset\\
    &\implies \abs{C_{g,r^*}} = 0 = \mathbb{I}[g \in C]
\end{align*}
\textbf{\textit{Case 2.}}
\begin{align*}
    g \in C &\implies \exists c \in C: \ g = c\\
    &\implies \exists c \in C: a(g, c) = r^*\\
    &\implies C_{g,r^*} \neq \emptyset\\
    &\implies \abs{C_{g,r^*}} > 0
\end{align*}
also
\begin{align*}
    c, c' \in C_{g,r} &\implies \left(a(g,c) = r^*\right) \land \left(a(g,c') = r^*\right) \\
    &\implies (g = c) \land (g = c')\\
    &\implies c = c'
\end{align*}
which shows there can only be at most one candidate in $C_{g,r^*}$, so $\abs{C_{g,r^*}} = 1 = \mathbb{I}[g \in C]$.\\
\\
We can then simplify
\begin{align*}
    \sum_{r \in R\backslash\set{r^*}} \abs{C_{g,r}} &= \sum_{r \in R} \abs{C_{g,r}} - \abs{C_{g,r^*}}\\
    &= \sum_{r \in R} \abs{C_{g,r}} - \mathbb{I}[g \in C]\\
    &= \abs{C} - \mathbb{I}[g \in C]
\end{align*}
where we re-use working from Subsection~\ref{subsec:prop1}.
\end{proof}
\begin{corollary}\label{co:lb_1_3}
    For any ${C \subseteq S}$, $LB_1(C) \leq LB_3(C)$.
\end{corollary}

\begin{proof}
Let $g^* = \mathrm{argmin}_{g \in \UG{C}}V_1(g,C)$, meaning that
\begin{align*}
    LB_3(C) = \min_{g \in \UG{C}} V_1(g, C) = V_1(g^*, C)
\end{align*}
We have two cases.\\
\\
\textbf{\textit{Case 1}}. If $g \not\in C$ then
\begin{align*}
    \sum_{r \in R\backslash\set{r^*}} \abs{C_{g^*, r}} = \abs{C} - \mathbb{I}[g \in C] = \abs{C} - 1
\end{align*}
implying by Lemma~\ref{lemma:A} that
\begin{align*}
    \left(\abs{C_{g^*, r}}\right)_{r \in R\backslash\set{r^*}} \in A(\abs{C})
\end{align*}
since $g^*$ creates at most $M$ non-empty splits of $C$, and so
\begin{align*}
    LB_3(C) &= V_1(g^*, C)\\
    &= \abs{C} + \sum_{r \in R\backslash\set{r^*}} d\left(\abs{C_{g^*,r}}\right)\\
    &= \abs{C} + \sum_{s \in \set{\abs{C_{g^*,r}} : \ r \in R\backslash\set{r^*}}} d\left(s\right)\\
    &\geq \abs{C} + \min_{(a_1,\dots,a_M) \in A(\abs{C})} d\left(a_i\right)\\
    &= d(\abs{C})\\
    &= LB_1(C)
\end{align*}
\textbf{\textit{Case 2}}.  If $g \in C$ then we borrow some of the work from Case~1.
\begin{align*}
    \sum_{r \in R\backslash\set{r^*}} \abs{C_{g^*, r}} = \abs{C} - \mathbb{I}[g \in C] = \abs{C}
\end{align*}
and so
\begin{align*}
    LB_3(C) &= V_1(g^*, C)\\
    &= \abs{C} + \sum_{r \in R\backslash\set{r^*}} d(\abs{C_{g^*, r}})\\
    &= -1 + (\abs{C} + 1) + \sum_{s \in \set{\abs{C_{g^*,r}} : \ r \in R\backslash\set{r^*}}} d\left(s\right)\\
    &\geq -1 + (\abs{C} + 1) + \min_{(a_1,\dots,a_M) \in A(\abs{C}+1)} d\left(a_i\right)\\
    &= -1 + d(\abs{C} + 1)\\
    &\geq d(\abs{C})\\
    &= LB_1(C)
\end{align*}
Hence in either case, $LB_1(C) \leq LB_3(C)$.
\end{proof}

Note that to reach the desired result of Theorem~\ref{th:lb_same}, we must also show the same for $LB_2(C)$ and $LB_4(C)$.

\begin{corollary}\label{co:lb_2_4}
For any ${C \subseteq S}$, $LB_2(C) \leq LB_4(C)$.
\end{corollary}

\begin{proof}
First we expand what Corollary~\ref{co:lb_1_3} states using Definition~\ref{def:min_total}.
\begin{align*}
    &\textsc{Bound}(\abs{C}, \MaxSplits{S})\\
    &\leq \min_{g \in \UG{C}} \abs{C} + \!\!\sum_{r \in R\backslash\set{r^*}} \!\!\textsc{Bound}(\abs{C_{g,r}}, \MaxSplits{S})
\end{align*}
This is true for any $C \subseteq S$. Rather, we may also say this is true for any $S \supseteq C$. Clearly $C$ is a super set of itself, so we may then state
\begin{align*}
    &\textsc{Bound}(\abs{C}, \MaxSplits{C})\\
    &\leq \min_{g \in \UG{C}} \abs{C} +\!\! \sum_{r \in R\backslash\set{r^*}}\!\! \textsc{Bound}(\abs{C_{g,r}}, \MaxSplits{C})
\end{align*}
The left hand side of the inequality is $LB_2(C)$, but the right hand side must be further worked. Note that by Lemma~\ref{lemma:maxsplits},
\begin{align*}
    &C_{g,r} \subseteq C \\
    &\implies \MaxSplits{C_{g,r}} \leq \MaxSplits{C}\\
    &\implies \textsc{Bound}(C_{g,r}, \MaxSplits{C_{g,r}}) \\
    &\quad \quad \quad \quad \geq \textsc{Bound}(C_{g,r}, \MaxSplits{C})\\
    &\implies LB_2(C_{g,r}) \geq \textsc{Bound}(C_{g,r}, \MaxSplits{C})
\end{align*}
and so
\begin{align*}
    &LB_2(C) \\
    &\leq \min_{g \in \UG{C}} \abs{C} + \!\!\!\!\sum_{r \in R\backslash\set{r^*}}  \!\!\!\!\textsc{Bound}(\abs{C_{g,r}}, \MaxSplits{C})\\
    &\leq \min_{g \in \UG{C}} \abs{C} +  \!\!\!\! \sum_{r \in R\backslash\set{r^*}}  \!\!\!\! LB_2(C_{g,r})\\
    &= LB_4(C)
\end{align*}
\end{proof}

We can now prove Theorem~\ref{th:lb_same}.\\
\\
\textbf{Theorem~\ref{th:lb_same}}. For any integer ${n \geq 1}$, we have that ${LB_{n}(C) \leq LB_{n+2}(C) \leq \MinTotal{C}}$ for any ${C \subseteq S}$.
\begin{proof}
By Corollary~\ref{co:lb_1_3} and Theorem~\ref{th:lb_adjacent}, we know that for all $C \subseteq S$
\begin{align*}
    LB_1(C) \leq LB_3(C) \leq \MinTotal{C}
\end{align*}
since Theorem~\ref{th:lb_adjacent} implies $LB_i(C) \leq \MinTotal{C}$ for any $i \geq 1$. We may then use Theorem~\ref{th:lb_adjacent} to show that the claim is true for any \textbf{odd} $n$. Repeating this work but using Corollary~\ref{co:lb_2_4} shows that the claim is true for \textbf{even} $n$.
\end{proof}

\subsection{Proof of Proposition~2}
\begin{proof}
Theorem~\ref{th:lb_adjacent} shows that $LB_n(C) \leq \MinTotal{C}$ for any integer $n \geq 1$ and any $C \subseteq S$. From this we may follow a similar proof to Lemma~\ref{lemma:lbi_j_v_i_j} to achieve the desired result.
\end{proof}

\subsection{Proof of Theorem~\ref{th:v_same}}
\begin{proof}
Theorem~\ref{th:lb_adjacent} and Lemma~\ref{lemma:lbi_j_v_i_j} imply that for any $n \geq 1$, we have $V_n(g,C) \leq V_{n+2}(g,C)$ for any $g \in G$ and $C \subseteq S$.\\
\\
We have already shown in Proposition~\ref{p:v_v*} that for any integer $n \geq 1$, $V_n(g, C) \leq V^*(g, C)$ for any guess ${g \in G}$ and any $C \subseteq S$, which then completes the desired result.
\end{proof}

\subsection{Proof of Theorem~\ref{th:v_ends}}
\begin{proof}
First we prove a similar statement about $LB$, that $LB_{2|C| + 1}(C) = \MinTotal{C}$ for any ${C \subseteq S}$. We do this by way of induction.\\
\\
The base case of $|C| = 0$ is trivial.\\
\\
If $\abs{C} = 1$, let $c$ be the single element of $C$.
\begin{align*}
    C_{c,r} \neq \emptyset &\implies c' \in C: \ c' \in C_{c,r}\\
    &\implies c \in C_{c,r}\\
    &\implies a(c,c) = r\\
    &\implies r^* = r
\end{align*}
By Definition~\ref{def:ug}, $\UG{C} = C$, and so
\begin{align*}
    LB_{2|C| + 1}(C) &= LB_3(C)\\
    &= \min_{g \in \UG{C}} V_1(g, C)\\
    &= V_1(c, C)\\
    &= |C| + \sum_{r \in R\backslash\set{r^*}} LB_1(C_{c,r})\\
    &= |C| + \sum_{r \in R\backslash\set{r^*}} LB_1(\emptyset)\\
    &= 1\\
\end{align*}
We repeat the same process for $\MinTotal{C}$.
\begin{align*}
    \MinTotal{C} &= \min_{g \in \UG{C}} V^*(g, C)\\
    &= V^*(c, C)\\
    &= |C| + \sum_{r \in R\backslash\set{r^*}} \MinTotal{C_{c,r}} \\
    &= |C| + \sum_{r \in R\backslash\set{r^*}} \MinTotal{\emptyset} \\
    &= 1
\end{align*}
so the claim is true if $|C| = 1$.\\
\\
Assume the claim is true for any $|C| \leq M$ for some integer $M \geq 1$, and
suppose we have some $C' \subseteq S$ where $|C'| = M + 1$. Then we have
\begin{align*}
LB_{2|C'| + 1}(C') = \min_{g \in \UG{C'}} |C'| +\!\! \sum_{r \in R\backslash\set{r^*}}\!\! LB_{2M+1}(C'_{g,r})
\end{align*}



Because we are only considering $g \in \UG{C'}$ and $\abs{C'} = M + 1 > 1$, Property~\ref{property:ug} implies that $C'_{g,r} \subset C'$ for any $r \in R$. This shows that $\abs{C'_{g,r}} \leq M$, and it follows from Theorem~\ref{th:lb_same} that
\begin{align*}
    LB_{2\abs{C'_{g,r}} + 1}(C'_{g,r}) \leq LB_{2M + 1}(C'_{g,r}) \leq \MinTotal{C'_{g,r}}
\end{align*}
and by the induction step
\begin{align*}
    LB_{2\abs{C'_{g,r}} + 1}(C'_{g,r}) = \MinTotal{C'_{g,r}}
\end{align*}
and so
\begin{align*}
    LB_{2M + 1}(C'_{g,r}) = \MinTotal{C'_{g,r}}
\end{align*}
We can use this to note that for any $g \in G$,
\begin{align*}
    V_{2M+1}(g, C') &= \abs{C'} + \sum_{r \in R\backslash\set{r^*}} LB_{2M+1}(C'_{g,r})\\
    &= \abs{C'} + \sum_{r \in R\backslash\set{r^*}} \MinTotal{C'_{g,r}}\\
    &= V^*(g, C')
\end{align*}
Lastly,
\begin{align*}
    LB_{2\abs{C'}+1} &= LB_{2M+3}(C')\\
    &= \min_{g \in \UG{C'}} V_{2M+1}(g,C')\\
    &= \min_{g \in \UG{C'}} V^*(g,C')\\
    &= \MinTotal{C'}
\end{align*}
We have shown that if $LB_{2|C| + 1}(C) = \MinTotal{C}$ for any $|C| \leq M$, then it is also true for any $C'$ where $|C'| = M + 1$. With the base cases and the induction step proved, we have shown the claim to be true for any $C \subseteq S$.
\end{proof}

\subsection{Proof of Theorem~\ref{th:ub}}
\begin{proof}
Let $C$ be an arbitrary candidate set. By Theorem~\ref{th:lb_adjacent}, we may note the following for any $g \in G$ where $i_g$ is odd:
\begin{align*}
    LB_{i_g}(C) \leq LB_{i_g + 1}(C) \leq \textsc{UB}(C)
\end{align*}
and by Lemma~\ref{lemma:lbi_j_v_i_j}, it is also true that
\begin{align*}
    V_{i_g + 1}(g,C) \geq V_{i_g}(g, C) \geq \textsc{UB}(C)
\end{align*}

So we may replace any odd $i_g$ with $i_g + 1$, making it even. Hence w.l.o.g.\ assume all $i_g$ are even. Define
\begin{align*}
    I = \max_{g \in G} i_g
\end{align*}
noting that $I$ must also be even. Theorem~\ref{th:v_same} lets us state for any $g \in G$ that
\begin{align*}
    V_I(g, C) \geq V_{i_g}(g,C)
\end{align*}
and so
\begin{align*}
    \textsc{UB}(C) &\leq \min_{g \in G} V_{i_g}(g,C)\\
    &\leq \min_{g \in G} V_{I}(g,C)\\
    &= LB_{I + 2}(C)\\
    &\leq \MinTotal{C}
\end{align*}
$\textsc{UB}(C)$ however was defined to be an upper bound for $\MinTotal{C}$, and so we get the desired result.
\end{proof}
\section{Examples}
This paper stated several definitions and theorems, all of which are meant to apply to any general guessing game. In this section, we provide concrete applications of some of these as to provide some clarity. We chose not to include these in the paper due to space constraints, and as we did not feel that they were crucial to understanding the paper's contributions.

Recall that we denote by $G_W$ the set of allowed guesses for Wordle (as in the original version), $S_W$ is the set of allowed secrets and $R_W$ are the possible Wordle responses (colours encoded as digit strings according to Definition~\ref{def:wordle}).

\subsection{Example of Definition~\ref{def:candidates}}
Suppose we have the following pairs of guesses and responses for Wordle:
\begin{itemize}
    \item $g_1 = \code{COILS}$, responded with $r_1 = \code{00010}$
    \item $g_2 = \code{ALPHA}$, responded with $r_2 = \code{01000}$
    \item $g_3 = \code{OMEGA}$, responded with $r_3 = \code{01100}$
\end{itemize}
Then the candidate set as defined as
\begin{align*}
    C =& \set{s \in S_W: a(\code{COILS}, s) = \code{00010}} \\
    &\cap \set{s \in S_W: a(\code{ALPHA}, s) = \code{01000}}\\
    &\cap \set{s \in S_W: a(\code{OMEGA}, s) = \code{01100}}\\
    =& \set{\code{LEMUR},\code{LUMEN},\code{MELEE}}
\end{align*}

\subsection{Example of Definition~\ref{def:strategy}}
A simple strategy to play Wordle would be ``guess the alphabetically first candidate''. It should be clear that following this strategy will eventually terminate the game. However, Definition~\ref{def:strategy} does not explicitly rule out ``obviously bad'' strategies such as ``guess \code{QAJAQ} regardless of the last response''; we still consider this is a valid strategy.

\subsection{Example of Definition~\ref{def:split}}
Suppose $C = \{$\code{COILS}, \code{DONUT}, \code{FINAL}, \code{MELEE}, \code{OMEGA}, \code{REALM}, \code{TITAN}, \code{TRIAD}$\}$. Guessing $g = \code{ALPHA}$ then creates the following splits
\begin{align*}
    C_{g, \code{00000}} &= \set{\code{DONUT}}\\
    C_{g, \code{00002}} &= \set{\code{OMEGA}}\\
    C_{g, \code{01000}} &= \set{\code{COILS}, \code{MELEE}}\\
    C_{g, \code{10000}} &= \set{\code{TITAN}, \code{TRIAD}}\\
    C_{g, \code{11000}} &= \set{\code{FINAL}, \code{REALM}}
\end{align*}
$C_{g,r}$ is empty for the other possible responses.

\subsection{Example of Definition~\ref{def:turns_needed}}
Suppose we follow the strategy of guessing the alphabetically first candidate on each turn. Suppose the secret is \code{SNAKE}. Under this strategy, the first guess we make is the alphabetically first secret (since $C = S$ at the start of the game). For Wordle this is \code{ABACK} and guessing this gives response $a(\code{ABACK}, \code{SNAKE}) = \code{00201}$. Filtering for candidates then leaves 13 candidates.

The first candidate is now \code{DRAKE}, giving response $a(\code{DRAKE}, \code{SNAKE}) = \code{00222}$, which leaves 5 possible candidates.

The first candidate is now \code{FLAKE}, giving response $a(\code{FLAKE}, \code{SNAKE}) = \code{00222}$, which leaves 4 possible candidates.

The first candidate is now \code{QUAKE}, giving response $a(\code{QUAKE}, \code{SNAKE}) = \code{00222}$, which leaves 3 possible candidates.

The first candidate is now \code{SNAKE}, giving response $a(\code{SNAKE}, \code{SNAKE}) = \code{22222}$ and we end the game, because this is $r^*$ for Wordle.

We needed to submit 5 guesses to receive the $r^*$ response, so the $\textsc{TurnsNeeded}(\sigma, \code{SNAKE}) = 5$, where $\sigma$ is the strategy of guessing the alphabetically first candidate at each turn.

\begin{table*}[!ht]
    \centering
    \begin{tabular}{lrrrrr}
    \toprule
    $g$ &  $\textsc{InSet}(g, S_W)$ &  $\textsc{MSS}(g, S_W)$ &  $\textsc{ESS}(g, S_W)$ &   $\textsc{Information}(g, S_W)$ &  $\textsc{MostParts}(g, S_W)$ \\
    \midrule
    \code{QAJAQ} &       1 &            1369 &      925.101 &  14898.98 &         -18 \\
    \code{XYLYL} &       1 &            1334 &      856.50 &  14417.32 &         -28 \\
    \code{ABACK} &       0 &             925 &      444.38 &  12292.99 &         -57 \\
    \code{ADIEU} &       1 &             284 &      123.70 &  10105.95 &         -80 \\
    \code{SALET} &       1 &             221 &       71.27 &   8572.31 &        -148 \\
    \code{RAISE} &       0 &             168 &       61.00 &   8502.78 &        -132 \\
    \code{SLATE} &       0 &             221 &       71.57 &   8538.30 &        -147 \\
    \code{TRACE} &       0 &             246 &       74.02 &   8578.78 &        -150 \\
    \code{CRATE} &       0 &             246 &       72.90 &   8571.84 &        -148 \\
    \bottomrule
    \end{tabular}
    \caption{Using existing valuations on Wordle}
    \label{tab:ex:vals}
\end{table*}

\subsection{Example of Definition~\ref{def:val:strat}}
We provide some example values for the valuations defined in Subsection~\ref{subsec:strats}, applied on Wordle. Results are shown in Table~\ref{tab:ex:vals}.

We can see that some of the ``obviously'' bad starting words such as \code{QAJAQ} and \code{XYLYL} all have high scores regardless of the valuation. The better words however have much noticeably lower scores; but depending on the valuation the choice for the `best' guess changes. For example, following the \textsc{MaxSizeSplit} valuation (abbreviated to \textsc{MSS}), the best guess in this list is \code{RAISE}. The \textsc{MostParts} valuation however would determine that \code{TRACE} is the best guess. Moreover we do see some instances of equal scores being assigned to different guesses; for a ranking system this is clearly not desirable.\\
\\
The initial guess to be submitted (from all of $G_W$) as determined by each valuation-based strategy is as follows:
\begin{align*}
    \sigma_\textsc{InSet}(S_W) &= \code{ABACK}\\
    \sigma_\textsc{MaxSizeSplit}(S_W) &= \code{AESIR}\\
    \sigma_\textsc{ExpSizeSplit}(S_W) &= \code{ROATE}\\
    \sigma_\textsc{Information}(S_W) &= \code{SOARE}\\
    \sigma_\textsc{MostParts}(S_W) &= \code{TRACE}
\end{align*}

\subsection{Example of Definition~\ref{def:val:combined}}
Refer back to Table~\ref{tab:ex:vals}. If we were to use the combined valuation $V = \tuple{\textsc{MostParts},\textsc{MSS}}$, then we would have
\begin{align*}
    V(\code{SALET}, S_W) &= (-148, 221)\\
    V(\code{CRATE}, S_W) &= (-148, 246)
\end{align*}
By lexicographical ordering then we would consider the \code{SALET} to be the better guess. Note that without the inclusion of the \textsc{MSS} valuation, both words would have been assigned the same score, and we'd default to choosing \code{CRATE} due to alphabetical ordering.

\subsection{Example of Definition~\ref{def:ug}}
If there are no candidates, then it should make sense that there's nothing worth guessing. If there's only one candidate, then that candidate is the only possible secret.\\
\\
Consider the example $C = \{$\code{COILS}, \code{OMEGA}, \code{REALM}$\}$. We check if $g = \code{ALPHA}$ is useful, by writing out the splits it creates:
\begin{align*}
    C_{g, \code{01000}} &= \set{\code{COILS}}\\
    C_{g, \code{00002}} &= \set{\code{OMEGA}}\\
    C_{g, \code{11000}} &= \set{\code{REALM}}
\end{align*}
There is more than 1 non-empty split, so $\code{ALPHA}$ is useful w.r.t to $C$.\\
\\
Repeating this for $g = \code{FUZZY}$, we only have one non-empty split:
\begin{align*}
    C_{g, \code{00000}} &=  \set{\code{COILS}, \code{OMEGA}, \code{REALM}}
\end{align*}
This means that guessing \code{FUZZY} gives no useful information in telling which of the candidates may be the secret. This example also demonstrates why Property~\ref{property:ug} is equivalent to Definition~\ref{def:ug}.

\subsection{Example of Definition~\ref{def:maxsplits}}
For Wordle's secret set $S_W$, $\textsc{MaxSplits}(S_W) = 150$. This is only achieved by guessing $g = \code{TRACE}$.

\subsection{Example of Lemma~\ref{lemma:maxsplits}}
It is known that $\textsc{MaxSplits}(S_W) = 150$. Let $C'$ be the subset of $S_W$ that only contains words that start with \code{A}. We can calculate then that $\textsc{MaxSplits}(C') = 49$.

\subsection{Example of Lemma~\ref{lemma:lb1_2}}
Note that from the Definition~\ref{def:lb1_2}, $LB_1(S_W) = LB_2(S_W)$, so this doesn't make for a good example. As such we demonstrate this on the subset of $S_W$ of only words that start with \code{J}. This leaves 20 candidates; call this subset $C$.
\begin{align*}
    LB_1(C) &= 39\\
    LB_2(C) &= 42
\end{align*}
The code run to calculate this is provided in a footnote of Section~\ref{subsec:combining_valuations}.

\subsection{Example of Theorems~\ref{th:lb_adjacent} and \ref{th:lb_same}}
We use the same subset $C$ as defined in the previous subsection. Normally, $\textsc{MinTotal}$ would be impossible to calculate for large candidate sets due to the recursive nature of Definition~\ref{def:min_total}. We can calculate this for $C$ since there are only 20 candidates.
\begin{align*}
    LB_3(C) &= 44\\
    LB_4(C) &= 44\\
    \textsc{MinTotal}(C) &= 44
\end{align*}
So for this example, it holds that
\begin{align*}
    LB_1(C) \leq LB_2(C) \leq \textsc{MinTotal}(C)\\
    LB_3(C) \leq LB_4(C) \leq \textsc{MinTotal}(C)
\end{align*}
It also holds that
\begin{align*}
    LB_1(C) \leq LB_3(C) \leq \textsc{MinTotal}(C)\\
    LB_2(C) \leq LB_4(C) \leq \textsc{MinTotal}(C)
\end{align*}

\end{document}